\documentclass{article}

\usepackage[round,comma]{natbib}
\usepackage{fullpage}
\usepackage[ruled]{algorithm2e}
\usepackage{amsmath}
\usepackage{amssymb}
\usepackage{amsthm}
\usepackage{graphicx}
\usepackage{subfigure}
\usepackage{dsfont}

\newtheorem{theorem}{Theorem}
\newtheorem{lemma}[theorem]{Lemma}

\newcommand{\E}[1]{\mathbb E \left[#1\right]}

\newcommand{\lr}[1]{\left(#1\right)}
\newcommand{\II}{\mathds 1}

\renewcommand{\cite}{\citep}

\title{Advice-Efficient Prediction with Expert Advice}

\author{Yevgeny Seldin$^{1,2,3}$ \and Peter Bartlett$^{1,2,4}$ \and Koby Crammer$^5$}

\date{\small $^1$Mathematical Sciences School, Queensland University of Technology, Brisbane, QLD, Australia\\
$^2$Department of Electical Engineering and Computer Sciences, UC Berkeley, Berkeley, CA, USA\\
$^3$Department of Computer Science, University College London, London, UK\\
$^4$Department of Statistics, UC Berkeley, Berkeley, CA, USA\\
$^5$Department of Electrical Engineering, The Technion, Haifa, Israel
}

\begin{document}
\maketitle

\begin{abstract}
Advice-efficient prediction with expert advice (in analogy to label-efficient prediction) is a variant of prediction with expert advice game, where on each round of the game we are allowed to ask for advice of a limited number $M$ out of $N$ experts. This setting is especially interesting when asking for advice of every expert on every round is expensive. We present an algorithm for advice-efficient prediction with expert advice that achieves $O\lr{\sqrt{\frac{N}{M}T\ln N}}$ regret on $T$ rounds of the game.
\end{abstract}

\section{Motivation}

We study the problem of prediction with expert advice in a setting, where we have a large set of experts $N$, but asking for advice of all experts on all rounds of the game is overly expensive. For example, the experts may be financial advisers designing investment portfolios for the stock market and getting the advice of each expert may be very expensive. Therefore, we would like to ask for advice of a smaller number $M \leq N$ of experts on each round (generally, $M \ll N$), but still be close to the best we could get if we would ask all experts for their advices. We call this setting advice-efficient prediction with expert advice in analogy to label-efficient prediction with expert advice \cite{CBL06}.

\section{Setting and Notations}

We work in prediction with expert advice setting \cite{CBL06}. We denote the action space by ${\cal X}$, the outcome space by ${\cal Y}$, and the loss function by $\ell: {\cal X} \times {\cal Y}\rightarrow [0,1]$ (for our analysis there is no need to assume that the loss is convex in the first parameter). The number of experts is denoted by $N$ and the experts are indexed by $h \in \{1,\dots,N\}$. On each round $i$ of the game each expert $h$ produces an advice $\psi_i^h \in {\cal X}$. On each round the player is allowed to ask for advice of a fixed number $M \leq N$ of experts. The player asks for advice and plays action $X_i \in {\cal X}$. The environment then reveals an outcome $y_i$ and the player suffers a loss $\ell(X_i, y_i)$ and the experts suffer losses $\ell(\psi_i^h, y_i)$. The goal of the algorithm is to minimize the regret defined as $\sum_{i=1}^t \ell(X_i, y_i) - \min_h \lr{\sum_{i=1}^t \ell(\psi_i^h, y_i)}$.

\section{Main Result}

We prove the following regret bound for the algorithm presented in Algorithm \ref{algo:SubExp} box.

\begin{algorithm}[t]
$\forall h$: $\hat L_0(h) = 0$.\\
\For{$i = 1,2,...$}{
Let
\[
q_i(h) = \frac{e^{-\eta_i \hat L_{i-1}(h)}}{\sum_{h'} e^{-\eta_i \hat L_{i-1}(h')}}.
\]
Sample one expert $H_i$ according to $q_i$. Get advice $\psi_i^{H_i}$.\\~\\
Play $X_i = \psi_i^{H_i}$.\\~\\
Observe nature outcome $y_i$ and suffer loss $L_i = \ell(X_i, y_i)$.\\~\\
Sample $M-1$ additional experts uniformly without replacement. Let $\II_i^h = 1$ if expert $h$ was sampled and $\II_i^h = 0$ otherwise. (For $H_i$ used in the definition of $X_i$ we have $\II_i^{H_i} = 1$.)\\~\\
Get advices $\psi_i^h$ for the experts sampled.\\~\\
\[
\forall h: L_i^h = \ell(\psi_i^h, y_i) \frac{1}{q_i(h) + (1 - q_i(h)) \frac{M-1}{N-1}} \II_i^h.
\]
\[
\forall h: \hat L_i(h) = \sum_{j=1}^i L_j^h. 
\]
}
\caption{Advice-efficient prediction with expert advice.}
\label{algo:SubExp}
\end{algorithm}

\begin{theorem}
\label{thm:SubExp}
The expected regret of Algorithm \ref{algo:SubExp} on $T$ rounds of the game satisfies:
\[
\E{\sum_{i=1}^T L_i} - \min_h \lr{\sum_{i=1}^T \ell(\psi_i^h, y_i)} \leq 2 \sqrt{\frac{N}{M} T \ln N}.
\]
\end{theorem}

The ``price'' that we pay for observing the advice of $M$ instead of all $N$ experts is multiplicative $\sqrt{\frac{N}{M}}$ term. The constant is identical to the constant in the ``simple'' analysis of exponentially weighted forecasters in \citet[Corollary 2.2]{CBL06} and slightly worse than the constant in the tighter analysis in \citet[Theorem 2.3]{CBL06} (we are loosing a $\sqrt{2}$ factor), but we can improve the constant using similar techniques. 

\section{Analysis}

The analysis is based on the following lemma, which follows from the analysis of EXP3 by \citet{Bub10}.

\begin{lemma}
\label{lem:Lsum}
For any $N$ sequences of random variables $L_1^h, L_2^h, \dots$ indexed by $h \in \{1,\dots,N\}$, such that $L_i^h \geq 0$, and any non-increasing sequence $\eta_1, \eta_2, \dots$, such that $\eta_i \geq 0$, for $q_i(h) = \frac{\exp \lr{-\eta_i \sum_{j=1}^{i-1} L_j^h}}{\sum_{h'} \exp \lr{-\eta_i \sum_{j=1}^{i-1} L_j^{h'}}}$ (assuming for $i=1$ the sum in the exponent is zero), for all $h^\star$ simultaneously we have:
\begin{equation}
\label{eq:Lsum}
\sum_{i=1}^T \sum_h q_i(h) L_i^h \leq \sum_{i=1}^T \frac{\eta_i}{2} \sum_h q_i(h) \lr{L_i^h}^2 + \frac{\ln N}{\eta_{_T}} + \sum_{i=1}^T L_i^{h^\star}.
\end{equation}
\end{lemma}

Now we are ready to prove Theorem \ref{thm:SubExp}.

\begin{proof}[Proof of Theorem \ref{thm:SubExp}]

We study $\sum_h q_i(h) L_i^h$ and $\sum_h q_i(h) \lr{L_i^h}^2$ for the case of our algorithm. We have:
\[
\E{L_i^h} = \ell(\psi_i^h, y_i).
\]
And we have:
\begin{equation}
\label{eq:EqL}
\E{\sum_h q_i(h) L_i^h} = \sum_h q_i(h) \E{L_i^h} = \sum_h q_i(h) \ell(\psi_i^h, y_i) = \E{L_i}.
\end{equation}
We also have:
\begin{align*}
\sum_h q_i(h) \lr{L_i^h}^2 &= \sum_h q_i(h) \lr{\ell(\psi_i^h, y_i) \frac{1}{q_i(h) + (1 - q_i(h)) \frac{M-1}{N-1}} \II_i^h}^2\\
&=\sum_h q_i(h) \ell(\psi_i^h, y_i)^2 \lr{\frac{1}{q_i(h) + (1 - q_i(h)) \frac{M-1}{N-1}}}^2 \II_i^h\\
&\leq \sum_h q_i(h) \lr{\frac{1}{q_i(h) + (1 - q_i(h)) \frac{M-1}{N-1}}}^2 \II_i^h.
\end{align*}
And from here:
\begin{align}
\E{\sum_h q_i(h) \lr{L_i^h}^2} &\leq \sum_h q_i(h) \lr{\frac{1}{q_i(h) + (1 - q_i(h)) \frac{M-1}{N-1}}}^2 \E{\II_i^h}\notag\\
&= \sum_h q_i(h) \frac{1}{q_i(h) + (1 - q_i(h)) \frac{M-1}{N-1}}\notag\\
&= \sum_h \frac{q_i(h) (N-1)}{q_i(h)(N - M) + M - 1}\notag\\
&\leq \frac{N}{M}.\label{eq:EqL2}
\end{align}
The proof of the last inequality is provided in Lemma \ref{lem:sum} the appendix.

By taking expectations of the two sides of \eqref{eq:Lsum} and substituting \eqref{eq:EqL} and \eqref{eq:EqL2} we obtain for all $h^\star$:
\[
\E{\sum_{i=1}^t L_i} \leq \frac{N}{M} \sum_{i=1}^t \frac{\eta_i}{2} + \frac{\ln N}{\eta_t} + \sum_{i=1}^t \ell(\psi_i^{h^\star}, y_i).
\]

Finally, taking $\eta_i = \sqrt{\frac{M \ln N}{i N}}$ completes the proof.
\end{proof}

\section{Easy Extensions}

The following extensions are easy to show:

\begin{enumerate}
	\item Since the variance of $L_i^h$-s is bounded by $(N-1)/(M-1)$ independently of time, it is easy to derive a high-probability result with similar guarantees.
	
	\item It is easy to show that the algorithm and analysis can be extended to adversarial multiarmed bandits, where we are allowed to reveal the loss of more than one action on each round (reward games can be translated to loss games via the transformation $\ell = 1 - r$, where $r \in [0,1]$ is the reward and $\ell \in [0,1]$ is the loss). Specifically, assume that in adversarial multiarmed bandit game with $K$ arms the player plays and suffers the loss of one action on each round, but then the player is allowed to observe the losses of $M - 1$ additional arms on the same round. Then, by identifying each arm with an expert that always predicts that arm, we can show that the regret of Algorithm \ref{algo:SubExp} is $O\lr{\sqrt{\frac{K}{M}T\ln K}}$. Interestingly, for $M > 1$ the variance of importance-weighted sampling is bounded by $(K-1)/(M-1)$ for all game rounds and it is possible to derive high-probability guarantees without additional smoothing in contrast to the EXP3.P algorithm.
\end{enumerate}

{\small
\bibliography{bibliography}
}

\appendix

\section{Lemma \ref{lem:sum}}

\begin{lemma}
\label{lem:sum}
For any probability distribution $q$ on $\{1,\dots,N\}$ and any $M \leq N$:
\begin{equation}
\label{eq:sum}
\sum_{h=1}^N \frac{q(h)(N-1)}{q(h)(N-M) + M - 1} \leq \frac{N}{M}.
\end{equation}
\end{lemma}

\begin{proof}
First, we show that the maximum of \eqref{eq:sum} is attained by the uniform distribution $q(h) = 1 / N$. The Lagrangian corresponding to minimization of \eqref{eq:sum} subject to $\sum_h q(h) = 1$ is:
\[
{\cal L}(q) = \sum_{h=1}^N \frac{q(h)(N-1)}{q(h)(N-M) + M - 1} + \lambda \lr{1 - \sum_h q(h)}.
\]
The first derivative of the Langrangian is:
\[
\frac{\partial {\cal L}}{\partial q(h)} = \frac{(N-1)(q(h)(N-M) + M - 1) - q(h)(N-M)(N-1)}{\lr{q(h)(N-M) + M - 1}^2} - \lambda = \frac{(N-1)(M-1)}{\lr{q(h)(N-M) + M - 1}^2} - \lambda.
\]
The important point is that the derivative depends only on single $h$ and, therefore, when we equate the derivative to zero the extremum is achieved when all $q(h)$ are equal. And, as a result, they are equal to $1/N$.

The second derivative is:
\[
\frac{\partial^2 {\cal L}}{\partial q(h)^2} = - \frac{2(N-M)(N-1)(M-1)}{\lr{q(h)(N-M) + M - 1}^3} \leq 0
\]
(note that for $M > 1$ and $N > M$ the inequality is strict; and for $M=1$ or $N=M$ it is easy to check that \eqref{eq:sum} holds) and the mixed partial derivatives $\frac{\partial^2 {\cal L}}{\partial q(h) \partial q(h')} = 0$. Therefore, $q(h) = 1/N$ is the maximum point (for $1 < M < N$). Substituting $q(h) = 1/N$ into \eqref{eq:sum} we get:
\[
\sum_{h=1}^N \frac{\frac{1}{N}(N-1)}{\frac{1}{N}(N-M) + M - 1} = \frac{N(N-1)}{N - M + N(M-1)} = \frac{N(N-1)}{M(N-1)} = \frac{N}{M}.
\]
\end{proof}

\end{document}